\pdfoutput=1
\documentclass[letterpaper, 10 pt, conference]{ieeeconf}

\IEEEoverridecommandlockouts                              

\makeatletter
\let\NAT@parse\undefined
\makeatother

\usepackage{amssymb,amsmath}

\usepackage[sort&compress,numbers,sectionbib]{natbib}

\usepackage{cite} 
\usepackage{booktabs}
\usepackage[font=small]{caption}
\usepackage{subcaption}
\usepackage{float}
\usepackage{graphicx}

\usepackage{verbatimbox}
\usepackage{multirow}
\usepackage{balance}
\usepackage{algorithm}
\usepackage{algorithmic}
\usepackage{comment}

\usepackage{dsfont}

\newcommand{\figref}[1]{Fig.~\ref{#1}}
\newcommand{\tabref}[1]{Table~\ref{#1}}

\usepackage[hyphens]{url}

\usepackage{hyperref}
\hypersetup{bookmarksopen,bookmarksnumbered,
pdfpagemode=UseOutlines,
colorlinks=true,
linkcolor=teal,
anchorcolor=teal,
citecolor=teal,
filecolor=teal,
menucolor=teal,
urlcolor=teal,
breaklinks=true
}
\newtheorem{theorem}{Theorem}[section]

\usepackage[dvipsnames]{xcolor}
\definecolor{darkblue}{RGB}{0.15,0.15,0.55}
\definecolor{lightgrey}{RGB}{0.75,0.75,0.75}
\providecommand{\codecomment}[1]{\textcolor{lightgrey}{\dotfill}\textcolor{darkblue}{//\,\textrm{#1}}}

\definecolor{myred}{RGB}{215,48,39}
\definecolor{myblue}{RGB}{69,117,180}
\definecolor{myorange}{RGB}{252,141,89}
\definecolor{mylightblue}{RGB}{145,191,219}

\definecolor{MYlightblue}{RGB}{217,95,2} 
\definecolor{MYdarkblue}{RGB}{117,112,179} 
\definecolor{MYgreen}{RGB}{27,158,119}

\usepackage{color,soul}
\usepackage{svg}

\newcommand{\jnote}[1]{{\xxnote{JA}{blue}{#1}}}
\newcommand{\xxnote}[3]{}
\ifx\hidenotes\undefined
  \usepackage{color}
  \renewcommand{\xxnote}[3]{\color{#2}{#1: #3}}
\fi

\usepackage[normalem]{ulem}
\useunder{\uline}{\ul}{}

\usepackage{bm}


\title{\LARGE \bf ReloPush: Multi-object Rearrangement in Confined Spaces\\ with a Nonholonomic Mobile Robot Pusher} 
\author{Jeeho Ahn and Christoforos Mavrogiannis\thanks{Department of Robotics, University of Michigan, Ann Arbor, USA. Email: \tt\small \{jeeho, cmavro\}@umich.edu}}

\begin{document}



\maketitle
\thispagestyle{empty}
\pagestyle{empty}

\begin{abstract}
We focus on push-based multi-object rearrangement planning using a nonholonomically constrained mobile robot. The simultaneous geometric, kinematic, and physics constraints make this problem especially challenging. Prior work on rearrangement planning often relaxes some of these constraints by assuming dexterous hardware, prehensile manipulation, or sparsely occupied workspaces. Our key insight is that by capturing these constraints into a unified representation, we could empower a constrained robot to tackle difficult problem instances by modifying the environment in its favor. To this end, we introduce a \emph{Push-Traversability} graph, whose vertices represent poses that the robot can push objects from, and edges represent optimal, kinematically feasible, and stable transitions between them. Based on this graph, we develop \emph{ReloPush}, a graph-based planning framework that takes as input a complex multi-object rearrangement task and breaks it down into a sequence of single-object pushing tasks. We evaluate ReloPush across a series of challenging scenarios, involving the rearrangement of densely cluttered workspaces with up to nine objects, using a 1/10-scale robot racecar. ReloPush exhibits orders of magnitude faster runtimes and significantly more robust execution in the real world, evidenced in lower execution times and fewer losses of object contact, compared to two baselines lacking our proposed graph structure.




\end{abstract}


\section{Introduction}\label{sec:introduction}

Autonomous mobile robots have revolutionized fulfillment by offering a robust and scalable solution for completing massive rearrangement tasks. Fulfillment sites tend to be highly structured, requiring extensive workspace engineering, like the installation of rails along rectilinear grids, and specialized docking mechanisms for handling packages. While effective, this approach can be prohibitively costly and impractical for many critical domains like construction, waste management, and small/medium-sized warehouses. These environments involve rearrangement tasks for a wide range of object geometries, require precise navigation among static and movable objects contrast, and impose practical constraints like respecting tight geometric boundaries and robot kinematics.




A practical technique to expand the diversity of objects that mobile robots can rearrange is pushing. Pushing is a form of \emph{nonprehensile manipulation}, a class of manipulation that does not require secure grasping but rather exploits the task mechanics to rearrange an object~\citep{lynch1999nonprehensile}. This is appealing because it enables the rearrangement of large, heavy, or unstructured objects even by simple mechanisms. However, pushing introduces constraints to robot motion: to ensure object stability, robots need to avoid abrupt turns and high accelerations.

\begin{figure}[!ht]
\begin{subfigure}{0.485\linewidth}
\includegraphics[width=\linewidth]{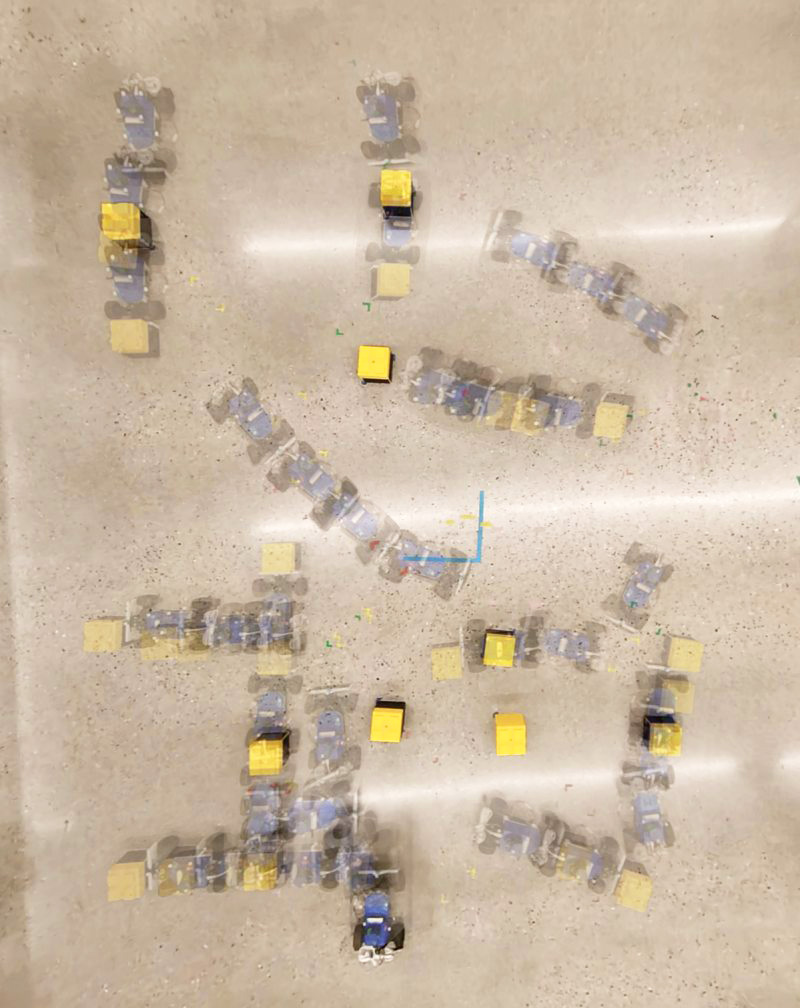}
\caption{Executing a rearrangement plan.}
\label{fig:relopush_before}
\end{subfigure}
\begin{subfigure}{0.485\linewidth}
\includegraphics[width=\linewidth]{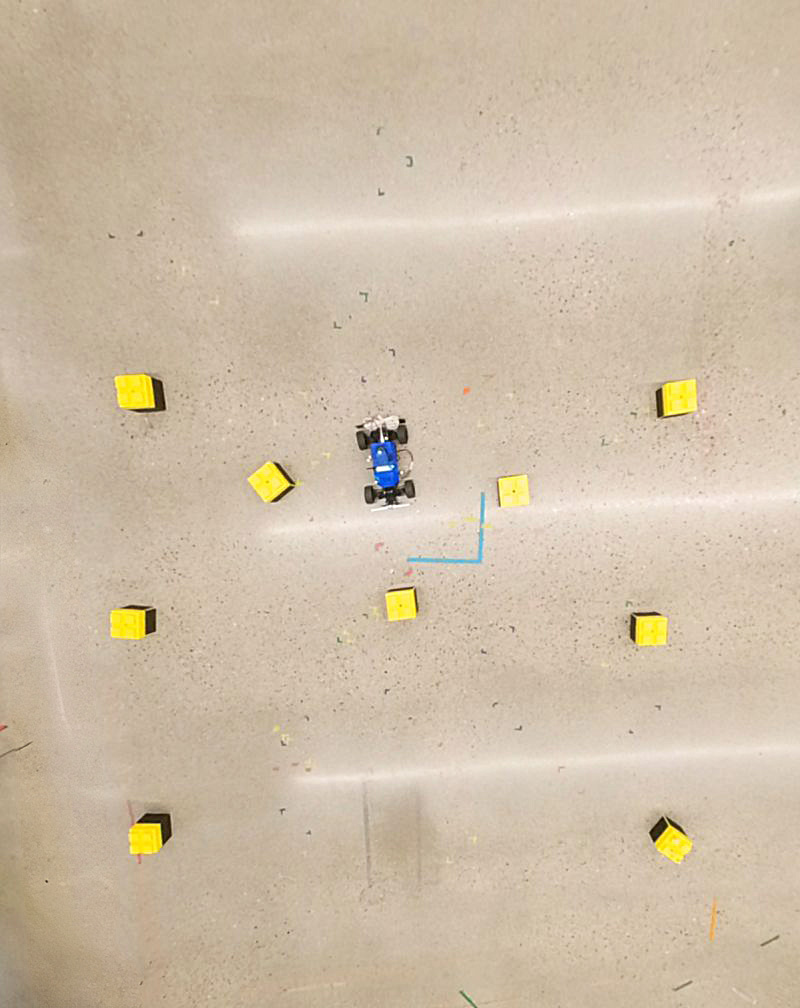}
\caption{Resulting rearrangement.}
\label{fig:relopush_after}
\end{subfigure}
  \caption{In this work, we describe ReloPush, a planning framework for tackling multi-object rearrangement tasks with a nonholonomic mobile robot pusher.
  }
  \label{fig:main}
\end{figure}

Real-world applications impose additional pragmatic constraints that need to be taken into account, i.e., geometric constraints imposed by the workspace boundary and by obstacles within, and robot kinematics (e.g., nonholonomic constraints). These constraints, in conjunction with push-stability constraints may render many practical rearrangement tasks infeasible. Prior work has tackled constrained rearrangement tasks by embracing the paradigm of \emph{planning among movable objects}~\citep{stilman2008planning,stilman2007manipulation}, in which the robot strategically modifies its environment to simplify planning. Previous applications emphasize the use of dexterous manipulators that are often capable of unconstrained overhand grasping, ignore orientation specifications on goal object poses, and assume generous workspace boundaries. These assumptions severely limit the potential of robot deployments for achieving practical productivity in real-world applications.

In this work, we focus on the problem of rearranging a set of objects into a set of desired final poses within a confined workspace via pushing, using a nonholonomic mobile robot pusher. Our key insight is that we could capture geometric, physics, and kinematic constraints into a unified representation that could enable the robot to understand when and how to modify the environment to complete its downstream rearrangement task. To this end, we introduce a \emph{push-traversability} graph whose edges represent kinematically feasible and stable object displacements. By planning on this graph, we achieve fast and effective rearrangement planning that can handle multiple objects in a confined space as shown in~\figref{fig:main}. Our code can be found at~\href{https://github.com/fluentrobotics/ReloPush}{https://github.com/fluentrobotics/ReloPush} and footage from our experiments at~\href{https://youtu.be/_EwHuF8XAjk}{https://youtu.be/\_EwHuF8XAjk}.





\section{Related Work}\label{sec:related-work}



\begin{figure*}[!ht]
    \centering
    \includegraphics[width=\linewidth]{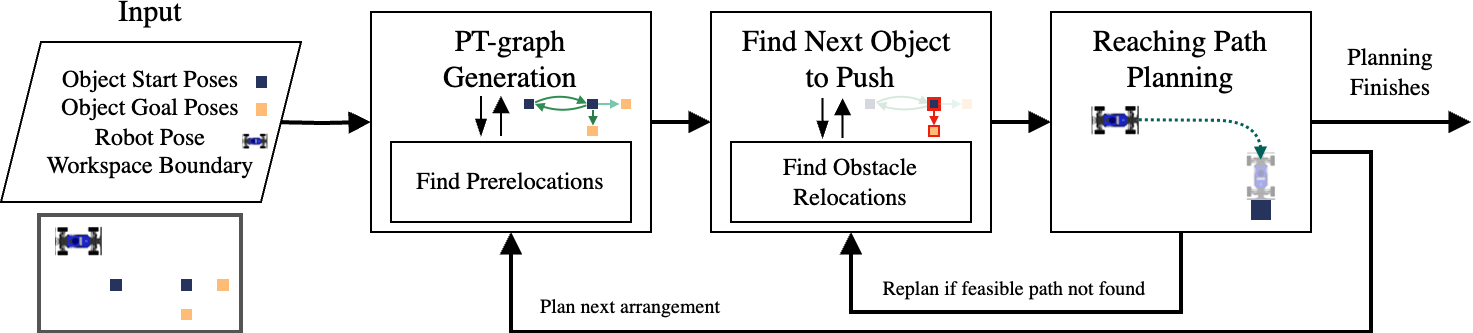}
    \caption{The ReloPush architecture. Given  the initial pose of the pusher and a rearrangement task in the form of start/goal object poses, ReloPush plans an efficient sequence of rearrangement subtasks to be executed by the robot via pushing. 
    }
    \label{fig:architecture}
\end{figure*}


\textbf{Planning among Movable Obstacles.} Many real-world environments include movable obstacles that robots can manipulate to free up space to make path planning feasible. This problem, first formulated by~\citet{wilfong1988motion} is PSPACE-hard when the final positions of objects are specified (these instances are known as~\emph{labeled}) and NP-hard otherwise (instances known as~\emph{unlabeled}). However, the relevance of the problem has motivated extensive investigation. \citet{chen1990practical} devised a hierarchical planner that proved effective in environments with polygonal obstacles. \citet{stilman2005navigation} demonstrated real-time planning among movable obstacles in realistic, cluttered households. Later approaches exhibited important properties such as probabilistic completeness~\citep{van2010path} and extensions to manipulation with high-DoF (degrees of freedom) arms~\citep{stilman2007manipulation,King-RSS-13,haustein2015kinodynamic}. 

We tailor the concept of planning among movable objects to a nonholonomic pusher, leveraging insights from Dubins path classification~\citep{shkel2001classification,lim2023circling}. In cases where the optimal rearrangement path violates the workspace boundary, we modify the object's starting pose to transition to a different optimal solution. When an obstacle is blocking an optimal object rearrangement, we remove the object to clear the way, by building on prior work on traversability graphs~\citep{nam2021fast}.


\textbf{Rearrangement Planning.} In rearrangement problems, the goal is to move a set of objects to (possibly predefined) goal poses. These are divided into two classes~\citep{pan2022algorithms}: \emph{monotone} instances, where each object only needs to be moved once, and \emph{non-monotone} ones, requiring more than one movement per object. Much of the prior work focuses on monotone instances~\citep{stilman2007manipulation,stilman2008planning,barry2013manipulation,talia2023pushr}, but real-world, densely cluttered spaces often give rise to non-monotone instances. While these have been shown to be NP-hard~\citep{huang2019large}, recent algorithms have demonstrated practical performance in manipulation tasks~\citep{han2018complexity,krontiris2015dealing,huang2019large,ren2022rearrangement,ren2024montecarlo}. While many works account also for geometric~\citep{ahn2023coordination,ahn2021integrated,nam2021fast} and kinodynamic constraints~\citep{ren2023kinodynamic}, most approaches make simplifying assumptions such as tasks without object orientation constraints, holonomic robots~\citep{ota2002rearrangement}, and high-DoF manipulators~\citep{krontiris2014rearranging,krontiris2015dealing,ren2024montecarlo}. 

We target monotone and non-monotone, labeled problem instances for a pusher that accounts for geometric, kinematic, and stability constraints. Through a novel \emph{push-traversability} graph whose edges incorporate all of these constraints, we manage to plan for challenging rearrangement instances in orders of magnitude lower runtimes than baselines. 

\textbf{Nonprehensile Rearrangement.} Some works tackle rearrangement tasks using nonprehensile manipulation. \citet{dogar2012planning} describe a planner that iteratively removes clutter via pushing to retrieve objects of interest in cluttered tabletops. \citet{King-RSS-13} develop a trajectory optimizer that pushes objects to simplify downstream manipulation tasks. \citet{huang2019large} use iterated local search to handle problems like singulation and sorting in densely cluttered tabletops. \citet{talia2023pushr} use multiagent pathfinding to distribute rearrangement tasks to a team of nonholonomic pushers. Some works focus on modeling the dynamics of pushing~\citep{Bauza2017,huang2022interleaving} to inform motion planning~\citep{haustein2015kinodynamic} whereas others learn adaptive pushing control policies~\citep{li2018push,yuan2018rearrangement} using data-driven techniques. For many domains, a practical assumption involves quasistatic pushing, for which analytical motion models exist~\citep{lynch1996stable,mason1986mechanics,goyal1991planar,howcutkosky}. Quasistatic pushing can empower simple control laws or even open-loop systems to perform robustly on many real-world problems.



While much of prior on nonprehensile rearrangement planning assumes high-DoF manipulators~\citep{King-RSS-13,huang2019large,Bauza2017,krontiris2014rearranging,krontiris2015dealing,ren2024montecarlo,haustein2015kinodynamic,stilman2007manipulation}, we demonstrate the practicality of quasistatic pushing on a nonholonomic mobile robot pusher~\citep{srinivasa2019mushr}. We move beyond prior work on nonholonomic nonprehensile rearrangement planning~\citep{talia2023pushr,king2016centric,talia2023pushr} by handling non-monotone problem instances in densely cluttered spaces (up to nine blocks of $0.15m$ side in a $4\times 5.2 m^2$ area). 

\section{Problem Statement}\label{sec:statement}


We consider a mobile robot \emph{pusher} and a set of $m$ polygonal \emph{blocks} in a workspace $\mathcal{W}\subset SE(2)$. 
We denote the state of the pusher as $p\in \mathcal{W}$ and the states of the blocks as $o_j\in \mathcal{W}$, $j\in\mathcal{M} = \{1,\dots, m\}$. 
The pusher follows rear-axle, simple-car kinematics $\dot{p} = f(p, u)$, where $u$ represents a control action (speed and steering angle), and may push objects using a flat bumper attached at its front (see~\figref{fig:exp_setup}). The goal of the pusher is to rearrange the blocks from their starting configuration, $O^s = (o_1^s,\dots,o_m^s)$, to a goal configuration, $O^g = (o_1^g,\dots,o_m^g)$. We seek to develop a planning framework to enable the pusher to efficiently rearrange all objects into their goal poses. We assume that the pusher has accurate knowledge of its ego pose at all times, and of the starting configuration of all objects, $O^s$.

\section{ReloPush: Nonprehensile Multi-object Rearrangement}\label{sec:approach}

We describe ReloPush, a planning framework for multi-object rearrangement via pushing. ReloPush breaks down a complex rearrangement task into an efficient sequence of single-object push-based rearrangement subtasks.

\subsection{System Overview}

Given a workspace $\mathcal{W}$, an initial robot pose $p_s$, and a set of objects that need to be reconfigured from their starting poses $O^s$ to their goal poses $O^g$, ReloPush finds a sequence of rearrangements in a greedy fashion. It first constructs a rearrangement graph (PT-graph) that accounts for robot kinematics, push stability, and workspace boundary constraints. Using graph search, ReloPush searches for the collision-free object rearrangement path of lowest cost. If such a path is found, the graph is updated to mark the rearranged object as an obstacle, and the planner is invoked again to find the next rearrangement of lowest cost. If the path found passes through a blocking object, ReloPush displaces that object out of the way first. If the path violates the workspace boundary, ReloPush displaces the object to be pushed until the path to its goal meets the boundary constraint. If the path is infeasible (i.e., fails to find a motion to approach the object to push), it replans with next rearrangement candidate that has the next lowest cost.
This process is repeated until a full rearrangement sequence for all objects is found. An overview of our architecture is shown in~\figref{fig:architecture}.


\begin{figure}
\begin{subfigure}{0.132\textwidth}
\includegraphics[width=\linewidth]{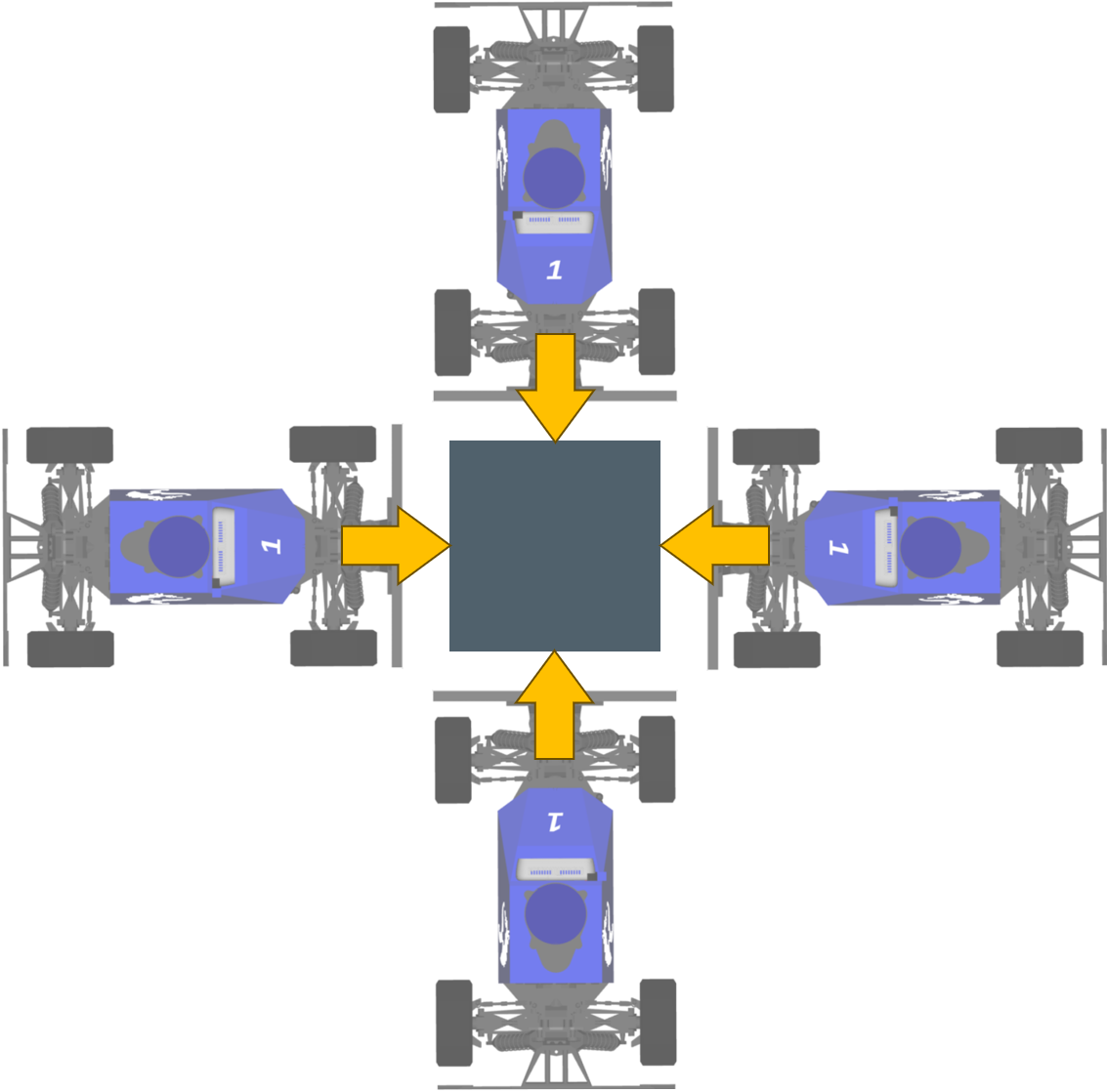}
\caption{Pushing poses}
\label{fig:ptgraph1}
\end{subfigure}
\begin{subfigure}{0.17\textwidth}
\includegraphics[width=\linewidth]{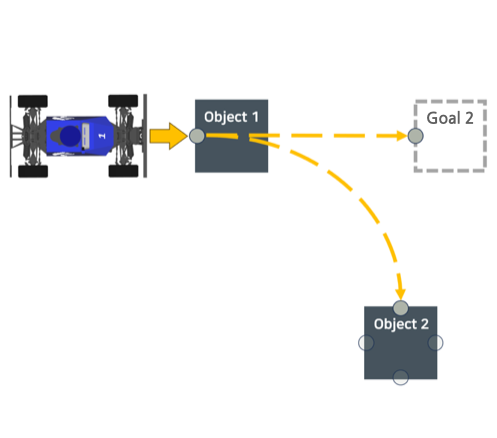}
\caption{Path Plan}
\label{fig:ptgraph2}
\end{subfigure}
\begin{subfigure}{0.165\textwidth}
\includegraphics[trim={0 30 0 0},clip,width=\linewidth]{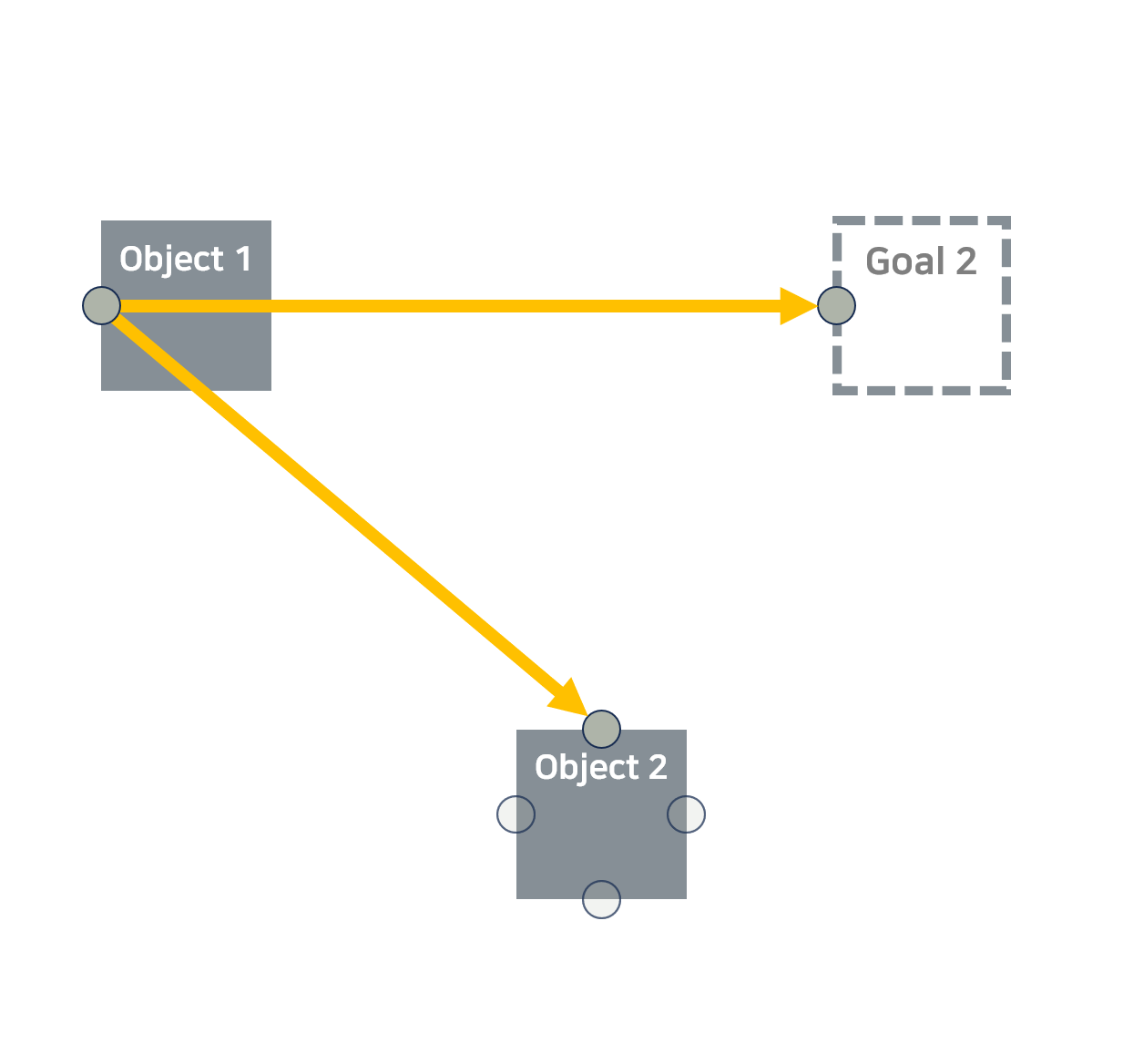}
\caption{PT-graph}
\label{fig:ptgraph3}
\end{subfigure}
  \caption{PT-graph generation. (\subref{fig:ptgraph1}) First, every object is assigned $K$ pushing poses (e.g., a cubic object has 4 pushing poses). (\subref{fig:ptgraph2}) For any pair of pushing poses, we check if a collision-free path that respects the steering limit for quasistatic pushing can be drawn. (c) For each valid path, we construct a directed edge between its start/goal vertices. }
  \label{fig:ptgraph}
\end{figure}

\subsection{Push-Traversability Graph}\label{sec:graph}

A \emph{traversability} graph (T-graph) is a representation of how movable objects can be reconfigured in a cluttered scene~\citep{nam2021fast}. In its original form, vertices represent (starting and goal) positions of objects and edges represent collision-free transitions between them. By searching the graph, a collision-free rearrangement plan can be found. 

Here, we build on the T-graph representation to introduce the \emph{push-traversability} graph (PT-graph) $G(V,E)$, which not only captures the spatial relationships among movable objects but also integrates the kinematic constraints of the pusher and push-stability constraints of objects within the edges. Because in push-based manipulation of polygonal blocks, the block orientation is important, each vertex in our graph ${v_i \in V}$ represents a valid robot \emph{pushing pose} $p_i$, i.e., a pose from which the pusher can start pushing a block (see~\figref{fig:ptgraph}). 

\begin{figure}
\begin{subfigure}{0.24\textwidth}
\includegraphics[width=\linewidth]{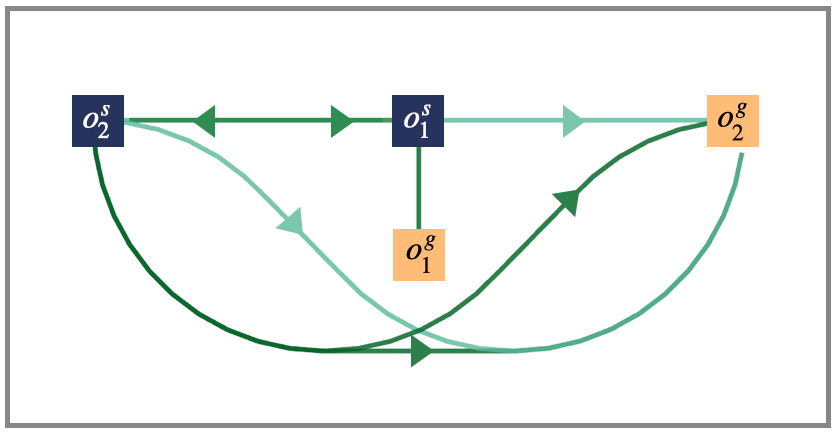}
\caption{Object traversability.}
\label{fig:tgraph1}
\end{subfigure}
\begin{subfigure}{0.24\textwidth}
\includegraphics[width=\linewidth]{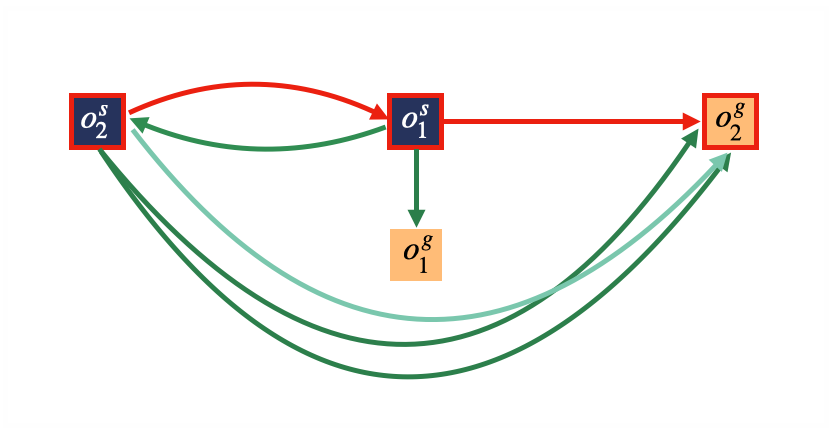}
\caption{PT-graph for the task in (\subref{fig:tgraph1}).}
\label{fig:tgraph2}
\end{subfigure}
  \caption{(\subref{fig:tgraph1}) Two objects (navy squares) need to be rearranged to goal poses (yellow squares). (\subref{fig:tgraph2}) The PT-graph: nodes are pushing poses and edges are Dubins paths connecting them. By searching the graph, we can determine if any blocking objects need to be removed. For instance, the initial pose of object 1 is found to be blocking the shortest rearrangement of object 2 (red path).}
  \label{fig:tgraph}
\end{figure}


For any pair of vertices $(v_s,v_g)$ representing a pair of start and goal \emph{pushing poses} $(p_s,p_g)$, a directed edge $e$ is formed from $v_s$ to $v_g$ if the optimal path from $p_s$ to $p_g$ is collision-free and within the workspace boundary. Optimality in the transitions is motivated by the heavily constrained problem domain which further demands efficient use of space. For a nonholonomically constrained mobile robot, the optimal path from $p_s$ to $p_g$ is a Dubins curve, and can be synthesized using L, R, and S primitives corresponding respectively to left, right, and straight motion~\citep{dubins1957curves}. To account for push stability, the L/R primitives are implemented using a minimum turning radius $\rho$ that ensures stable pushing under the quasistatic assumption~\citep{King-RSS-13,talia2023pushr,lynch1996stable,goyal1991planar}. If the Dubins curve is collision-free and within the boundary, a directed edge is constructed from $v_s$ to $v_g$, and assigned a weight that is equal to the length of the curve. The edge direction is dictated by pushing stability constraints (e.g., contact cannot be maintained if the pusher moves backwards). Alg.~\ref{alg:gengraph} describes the graph construction and \figref{fig:tgraph} shows an example.



\subsection{Prerelocation: Change of Starting Pushing Pose\label{sec:prerelocation}}
 


Often, an edge between two vertices cannot be formed because the Dubins curve connecting them violates the workspace boundary. This is especially common due to the limited turning radius $\rho$ imposed by the push stability constraint (see~\figref{fig:example}). Our insight is that a small change in the starting pushing pose might allow for an optimal, collision-free rearrangement that lies entirely within the workspace boundary. To this end, we leverage prior work on the classification of Dubins curves~\citep{shkel2001classification,lim2023circling}. Intuitively, because of the robot's kinematic constraints, if the start and goal poses are ``too close'', the Dubins curve connecting them will tend to require a sequence of wide turns that violate the workspace boundary. In particular, if the Euclidean distance $d$ between the start and goal poses, normalized by the turning radius $\rho$, is smaller than a threshold $d_{th}$, the Dubins curve connecting them is a \textit{Short} path that will likely include an excessive turning arc. If $d>d_{th}$, the corresponding Dubins curve is a \emph{Long} path that will likely not require excessive wide turning. This threshold can be found to be $d_{th} = \lvert \sin\alpha \rvert + \lvert \sin\beta \rvert + \sqrt{4-(\cos\alpha + \cos\beta)^2}$, where $\alpha$ and $\beta$ represent the start and goal orientations with respect to the line connecting them and transformed to be horizontal~\citep{shkel2001classification,lim2023circling}. \figref{fig:example} shows two paths starting from different poses and leading to the same goal pose via a \textit{Long} and a \textit{Short} path. By moving the object's starting pose slightly towards the bottom, a \textit{Long} Dubins curve leading to the goal pose can be found to be within bounds.




\begin{figure}
    \captionsetup{skip=0pt}
    \centering
    \includegraphics[width=\linewidth]{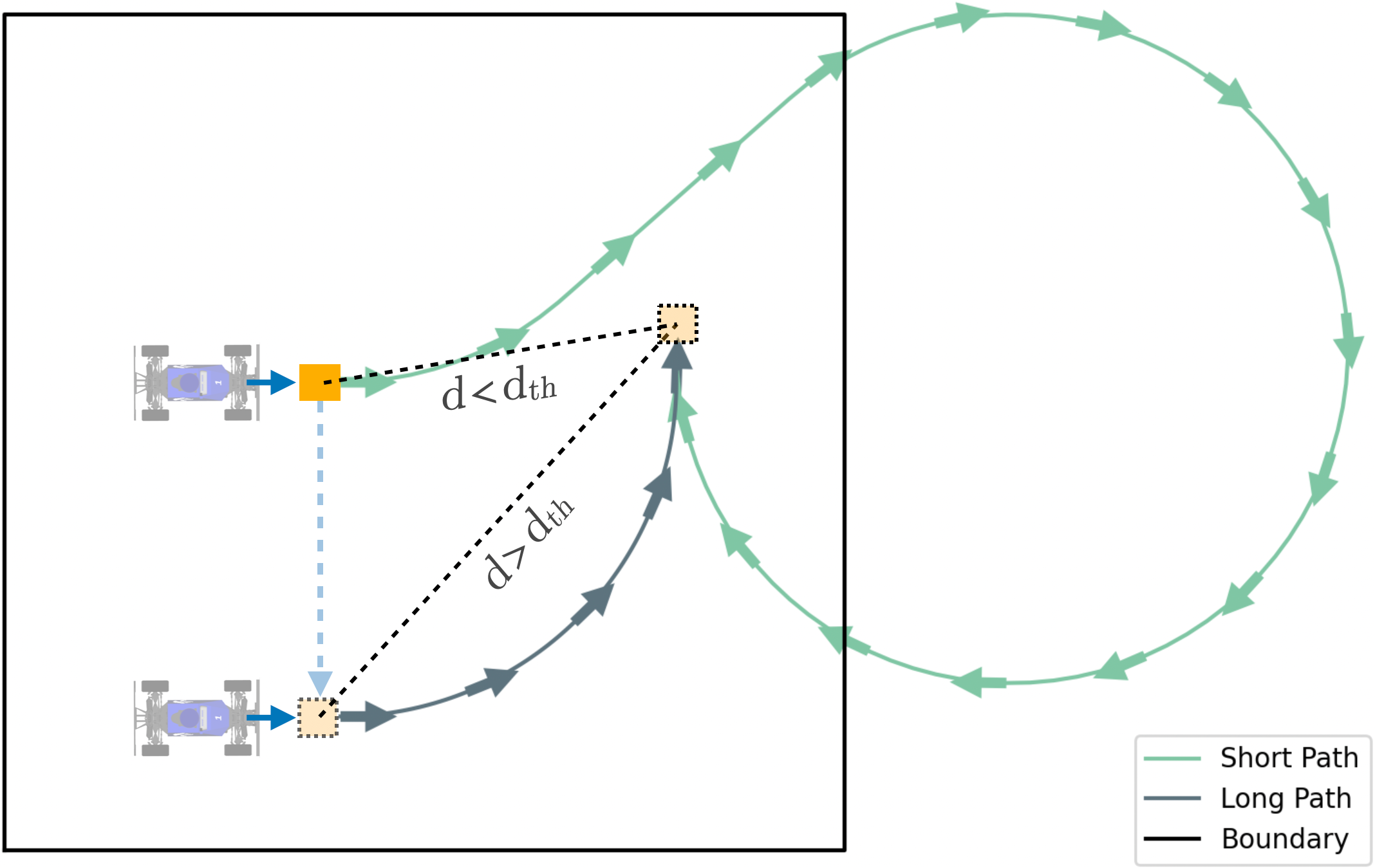}
    \caption{Two Dubins curves with the same goal pose (top right) and maximum turning radius. When the start pose is too close to the goal ($d\leq d_{th}$), the resulting Dubins curve (green color) is a \emph{Short Path} involving large turns violating the workspace boundary. Using Dubins path classification~\citep{shkel2001classification,lim2023circling}, we can determine a \textit{prerelocation} of the object's starting pose to allow reaching the goal via a \emph{Long Path} ($d> d_{th}$) which will involve smaller turning arcs (gray color).}
    \label{fig:example} 
\end{figure}


For instances where an object rearrangement path is found to be out-of-bounds and meet the condition of a \textit{Short} path (with excessive turning, i.e., $d\leq d_{th}$), we attempt to find another start pose $p_s'$ that can be connected to $p_g$ via a collision-free \textit{Long} Dubins curve and that is reachable by the robot from $p_s$. We refer to this change of starting pose from $p_s$ to $p_s'$ as a \textit{Prerelocation}. To find $p_s'$, we evenly sample configurations along the pushing directions of the object (see~\figref{fig:ptgraph1}), and choose a configuration that meets the conditions above, while requiring minimal displacement $||p_s-p_s'||$ from its initial configuration. This process is abstracted as the function \textsc{FindPreRelo} in Alg.~\ref{alg:gengraph}.







\begin{algorithm}
\caption{\textsc{GenGraph}}
\label{alg:gengraph}
{\footnotesize
\textbf{Input:}
Object start poses $O^s$, Object goal poses $O^g$, Workspace $\mathcal{W}$, Min. Turning Radius $\rho$ \\
\textbf{Output:} Push-Traversability Graph $G$\\

1 \ $\mathcal{V} = \mbox{\textsc{GenVertices}}(O^s,O^g)$ \\
\ \codecomment{construct a vertex for each pushing pose for each object start/goal} \\
2 \ $\mathcal{E} = \varnothing$ \codecomment{start with no edge} \\
3 \ \textbf{for each} $v_s$ in $\mathcal{V}$ \\
4 \quad \textbf{if} $v_s$ is vertex of $O^g$ \\
5 \quad \quad \textbf{continue} \codecomment{does not consider push from goals} \\
6 \quad \textbf{else} \\
7 \quad \quad \textbf{for each} $v_g$ in $\mathcal{V}, O_{v_g} \neq O_{v_s}$ \codecomment{vertex of another object} \\
8 \quad \quad \quad $\mathcal{P}_{Dubins}$ = $\mbox{\textsc{ValidDubins}}(v_s,v_g,\rho,\mathcal{W})$ \\
\quad \quad \quad \codecomment{find a Dubins curve from $v_s$ to $v_g$ and check if it is collision-free} \\
\quad \quad \quad \codecomment{and within $\mathcal{W}$} \\
9 \quad \quad \quad \textbf{if} $\mathcal{P}_{Dubins}$ \codecomment{if the Dubins curve is valid} \\
10 \quad \quad \quad \quad $\mathcal{E} \leftarrow \mathcal{E} \cup (v_s,v_g,\mathcal{P}_{Dubins})$ \codecomment{$||\mathcal{P}_{Dubins}||$ to be used as the weight} \\
11 \quad \quad \quad \textbf{else} \\
12 \quad \quad \quad \quad \textbf{if} $\mathcal{P}_{Dubins}$ is out of bounds\\
13 \quad \quad \quad \quad \quad $p_{pre} = \mbox{\textsc{FindPreRelo}}(v_s,v_g,\mathcal{W})$ \\
\quad \quad \quad \quad \quad \codecomment{find different start pose near $v_s$ that induces a valid Dubins curve} \\
14 \quad \quad \quad \quad \quad \textbf{if} $p_{pre}$ \codecomment{a prerelocation made the push valid} \\
15 \quad \quad \quad \quad \quad \quad $\mathcal{P}'_{Dubins}$ = $\mbox{\textsc{ValidDubins}}(v_{pre},v_g,\rho,\mathcal{W})$ \\
16 \quad \quad \quad \quad \quad \quad $\mathcal{E} \leftarrow \mathcal{E} \cup (v_s,v_g,v_{pre})$ \codecomment{$||\mathcal{P}'_{Dubins}||$ is new path weight} \\
17 \quad \quad \quad \quad \quad \textbf{end if} \codecomment{pre-relocation}\\
18 \quad \quad \quad \quad \textbf{end if} \codecomment{out of bound}\\
19 \quad \quad \quad \textbf{end if} \codecomment{validity of dubins curve}\\
20 \quad \quad \textbf{end for} \codecomment{$v_g$}\\
21 \quad \textbf{end if} \codecomment{$v$ not in $O^g$} \\
22 \ \textbf{end for} \codecomment{$v_s$}\\
23 \ \textbf{return} $G(V,E)$

}\end{algorithm}

\subsection{Removing Blocking Objects}\label{sec:blocking}

Extracting a rearrangement path plan can be done by searching the PT-graph using any graph search algorithm. The extracted path may include a vertex that is different from the start and goal vertex. If that is the case, then that vertex corresponds to an object that is physically blocking the rearrangement path. This object needs to be displaced before the plan can be executed. To do so, we follow a similar technique to how we plan \textit{Prerelocations}: we find the closest relocation along the object's pushing directions (see~\figref{fig:ptgraph1}) that unblocks the path execution. This method of finding what object to remove is shown to be complete~\citep{nam2021fast}.


\subsection{Reaching to Push an Object}

To execute a rearrangement plan, the pusher needs to navigate between objects. To do that, we invoke a motion planner to check if there is a collision-free path connecting the robot's pose with a pushing pose. For every object rearrangement, there is at least one motion plan invocation to find a motion to approach an object. If a rearrangement involves a \textit{prerelocation} or the removal of a \textit{blocking object}, an additional motion plan is invoked. Any motion planner could be used but we found convenient to use Hybrid A*~\citep{dolgov2008practical}.







\begin{figure*}[ht]
    \centering
    \makebox[\textwidth][c]{
        \begin{subfigure}[b]{0.18\textwidth}
            \centering
            \includegraphics[width=\textwidth]{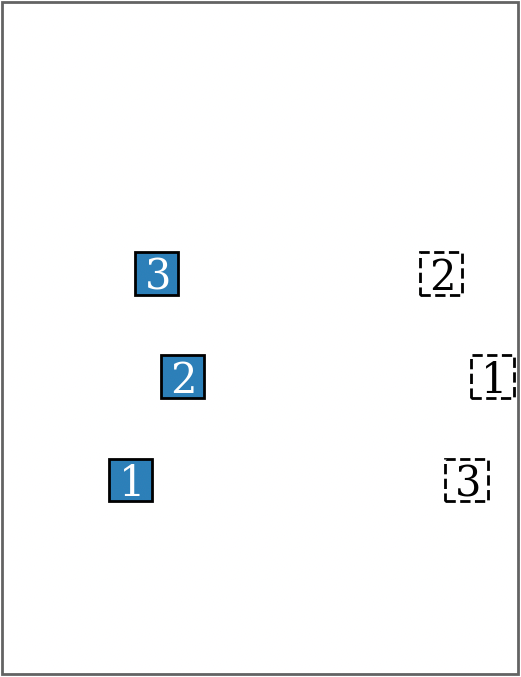}
            \caption{$m=3$}
        \end{subfigure}
        \hfill
        \begin{subfigure}[b]{0.18\textwidth}
            \centering
            \includegraphics[width=\textwidth]{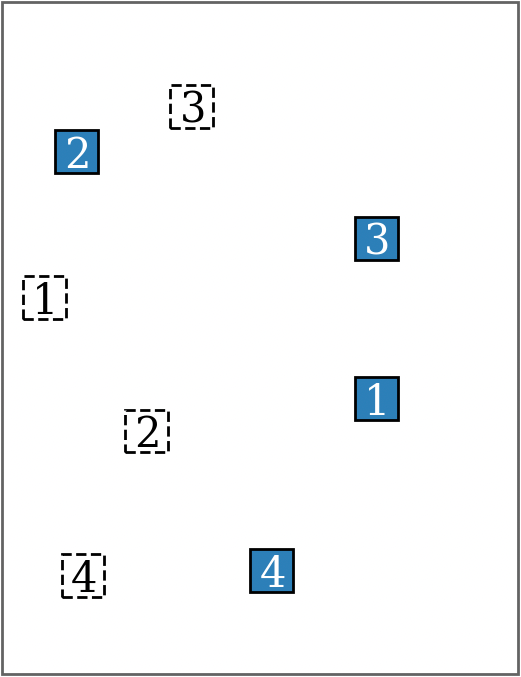}
            \caption{$m=4$}
        \end{subfigure}
        \hfill
        \begin{subfigure}[b]{0.18\textwidth}
            \centering
            \includegraphics[width=\textwidth]{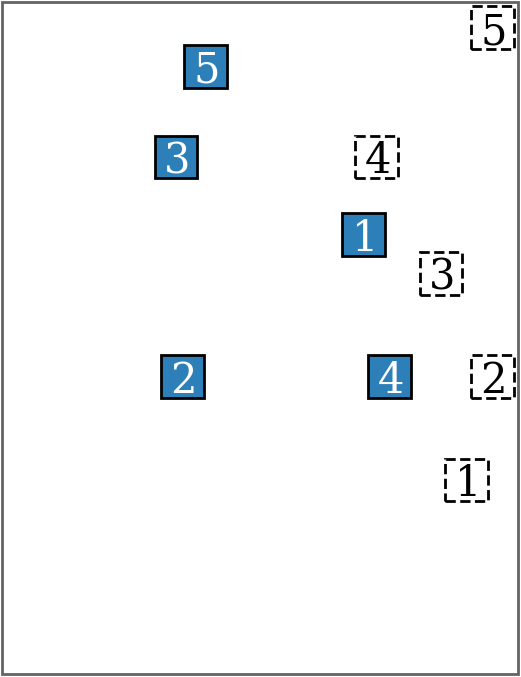}
            \caption{$m=5$}
        \end{subfigure}
        \hfill
        \begin{subfigure}[b]{0.18\textwidth}
            \centering
            \includegraphics[width=\textwidth]{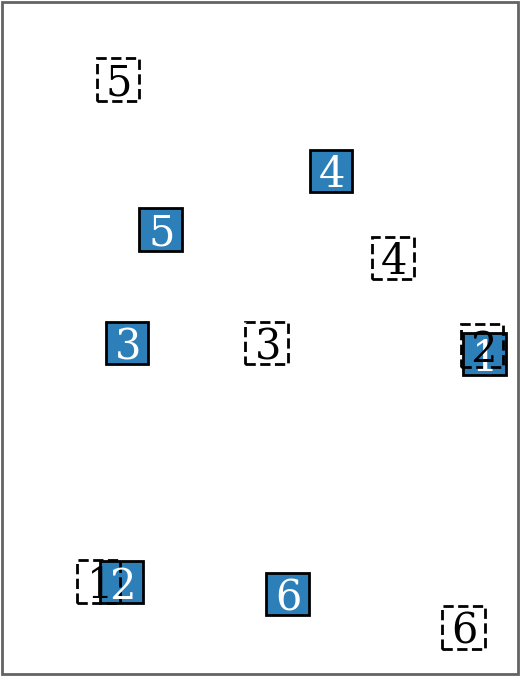}
            \caption{$m=6$}
        \end{subfigure}
        \hfill
        \begin{subfigure}[b]{0.18\textwidth}
            \centering
            \includegraphics[width=\textwidth]{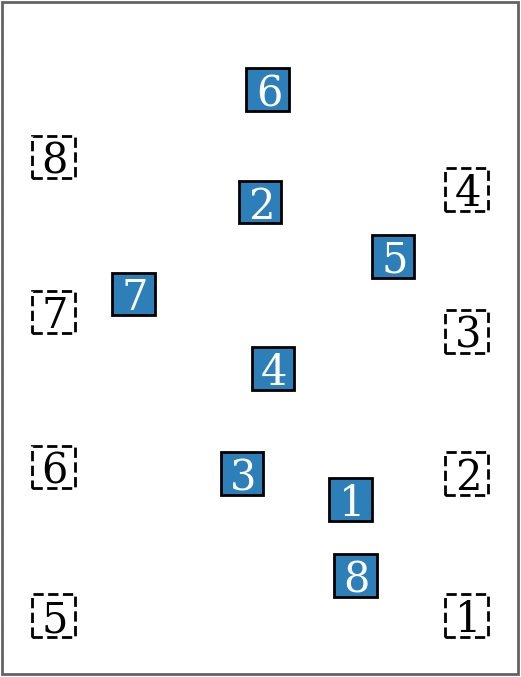}
            \caption{$m=8$}
        \end{subfigure}
    }
    \caption{Evaluation scenarios. Solid squares represent starting object poses and dashed squares represent goal poses.
    \label{fig:scenarios}}
\end{figure*}

\subsection{Analysis of the Algorithm}


\begin{theorem}
Assuming a bounded number of pushing poses per object, $K_{max}$, the graph construction runs in polynomial time.
\end{theorem}

\begin{proof}
The number of vertices per object is bounded by $K_{max}$, thus, a fully connected graph in our case has $K_{max}\cdot m$ vertices. For each edge, a Dubins path is found in $O(1)$ and its collision checking is done in $O(1)$ assuming a bounded number of configurations checked due to our confined workspace. Searching a graph with Dijkstra's algorithm runs in $O(|V|^2)$ in a directed complete graph (the number of edges dominates). Thus, the runtime for a the rearrangement of $m$ objects reduces to $O(m^3)$. To plan motion to approach an object, we invoke Hybrid A*, whose runtime also reduces to a polynomial expression on the number of objects assuming fixed workspace discretization, resolution of driving directions, and replanning attempts.
\end{proof}

\begin{theorem}
Alg.~\ref{alg:gengraph} is complete if \textsc{FindPreRelo} is complete. 

\end{theorem}

\begin{proof}
The collision checking for each Dubins path is complete since it always finds a collision if one exists by checking all grid cells that the path occupies. Thus, Alg.~\ref{alg:gengraph} is complete since it constructs the graph $G$ after a finite number of iterations, assuming that \textsc{FindPreRelo} is complete.
\end{proof}

\begin{theorem}
\textsc{ReloPush} is complete if Alg.~\ref{alg:gengraph} is complete and at least one object pose is reachable by the robot.


\end{theorem} 

\begin{proof}
By construction, an edge $e \in E$ between a pair $(v_s, v_g)$ represents a feasible path from $p_s$ to $p_g$. Thus, any path on $G$ is a collision-free path within the workspace $\mathcal{W}$. Therefore, for any path on G, if the robot can plan a path to its starting vertex, a complete rearrangement is feasible. 
\end{proof}





~

\begin{table*}
\captionsetup{skip=0pt}
\caption{Planning performance. Each cell lists the mean and the standard deviation over 100 trials per scenario.\label{tab:sim-results}}%
\label{tab:sim_result}
\centering
\resizebox{\linewidth}{!}{%
\begin{tabular}
{l|c|c|c|c|c|c|c|c|c|c|c|c|c|c|c}
\toprule
 \textbf{Scenario} & \multicolumn{3}{c|}{$m = 3$}& \multicolumn{3}{c|}{$m = 4$} & \multicolumn{3}{c|}{$m = 5$} & \multicolumn{3}{c}{$m = 6$} & \multicolumn{3}{|c}{$m = 8$} \\
\hline
\textbf{Algorithm} & \textsc{ReloPush} & \textsc{NPR} & \textsc{MP}  &  \textsc{ReloPush} & \textsc{NPR} & \textsc{MP}  &  \textsc{ReloPush} & \textsc{NPR} & \textsc{MP}  & \textsc{ReloPush} & \textsc{NPR} & \textsc{MP}  &
\textsc{ReloPush} & \textsc{NPR} & \textsc{MP}  \\
\midrule

$S$ (\%) 
& \textbf{100} & 55 & 52   
& \textbf{100} & \textbf{100} & 73 
& \textbf{80} & 64 & 56
& \textbf{89} & 25 & 3 
& \textbf{86} & 12 & 6\\
\hline


$T_p$ (ms) 
& \textbf{40} (4.3) & 1864 (141.1) & 465 (67.4) 
& \textbf{86} (3.5) & 5376 (131.5) & 956 (36.0) 
& \textbf{146} (6.8) & 9034 (581.5) & 1175 (117.6) 
& \textbf{318} (20.8) & 14489 (682.3)& 1487 (57.4)
& \textbf{529} (23.5) & 25654 (2205.6) & 2073 (149.5) \\

\hline






$L_t$ (m) 
& \textbf{32.3} (3.1) & 42.1 (5.2) & 38.7 (7.3) 
& \textbf{40.3} (0.8) & 45.5 (2.5) & 45.1 (2.5) 
& \textbf{48.3} (3.5) & 60.8 (6.5) & 61.0 (9.1) 
& 77.6 (5.5) & 74.5 (15.2) & \textbf{62.9 (1.0)}
& \textbf{90.3} (6.0) & 105.8 (4.2) & 101.6 (3.5)\\
\bottomrule

\end{tabular}
}
\end{table*}

\begin{table*}[htbp]
\captionsetup{skip=0pt}
\caption{Planning behavior. Each cell lists the mean and the standard deviation over 100 simulated trials per scenario.\label{tab:planning-behavior}}%
\label{tab:sim_result_relo}
\centering
\resizebox{\linewidth}{!}{%
\begin{tabular}{l|c|c|c|c|c|c|c|c|c|c|c|c|c|c|c}
\toprule
\textbf{Scenario} & \multicolumn{3}{c|}{$m = 3$} & \multicolumn{3}{c|}{$m = 4$} & \multicolumn{3}{c|}{$m = 5$} & \multicolumn{3}{c|}{$m = 6$} & \multicolumn{3}{c}{$m = 8$} \\
\midrule
Algorithm & \textsc{ReloPush} & \textsc{NPR} & \textsc{MP}  &  \textsc{ReloPush} & \textsc{NPR} & \textsc{MP}  &  \textsc{ReloPush} & \textsc{NPR} & \textsc{MP}  & \textsc{ReloPush} & \textsc{NPR} & \textsc{MP}  &
\textsc{ReloPush} & \textsc{NPR} & \textsc{MP} \\
\midrule
$N_{pre}$ & 0.9 (0.29) & - & - & 1.0 (0.00) & - & - & 1.6 (0.49) & - & - & 3.0 (0.45) & - & - & 4.4 (0.59) & - & -\\
\hline
$N_{obs}$ & 0.0 (0.00) & 0.0 (0.00) & - & 0.0 (0.00) & 0.0 (0.00) & - & 0.5 (0.52) & 1.0 (0.25) & - & 1.0 (0.00) & 0.5 (0.50) & - & 0.6 (0.49) & 0.0 (0.00) & -\\
\hline
$L_p$ (m) 
& \textbf{8.4} (0.1) & 19.1 (4.4) & 17.8 (4.3) 
& \textbf{8.5} (0.10) & 18.1 (0.8) & 17.3 (0.9) 
& \textbf{11.2} (0.5) & 29.6 (10.1) & 37.6 (7.2) 
& \textbf{13.9} (0.9) & 37.2 (7.8) & 35.6 (0.4)
& \textbf{23.7} (0.76) & 59.3 (3.1) & 55.9 (4.0)\\
\bottomrule
\end{tabular}
}
\end{table*}

\begin{table*}
\captionsetup{skip=0pt}
\caption{Results from real-world trials. Each cell lists the mean and the standard deviation over 5 trials per scenario.\label{tab:real-results}}%
\label{tab:real_result}
\centering
\resizebox{\linewidth}{!}{%
\begin{tabular}
{l|c|c|c|c|c|c|c|c|c|c|c|c|c|c|c}
\toprule
 \textbf{Scenario} & \multicolumn{3}{c|}{$m = 3$}& \multicolumn{3}{c|}{$m = 4$} & \multicolumn{3}{c|}{$m = 5$} & \multicolumn{3}{c|}{$m = 6$} & \multicolumn{3}{c}{$m = 8$}\\
\hline
\textbf{Algorithm} & \textsc{ReloPush} & \textsc{NPR} & \textsc{MP} & 
\textsc{ReloPush} & \textsc{NPR} &  \textsc{MP}  &  
\textsc{ReloPush} & \textsc{NPR} & \textsc{MP}  & 
\textsc{ReloPush} & \textsc{NPR} & \textsc{MP}  &
\textsc{ReloPush} & \textsc{NPR} & \textsc{MP}  \\
\midrule

$N_{loss}$
& \textbf{0.0} (0.0) & 0.6 (0.49) & 1.0 (0.0)  
& \textbf{0.0} (0.0) & 0.6 (0.49) & 1.0 (0.63) 
& \textbf{0.0} (0.0) & 0.4 (0.49) & 0.8 (0.75) 
& \textbf{0.2} (0.40) & 1.4 (0.49) & 1.2 (0.40)
& \textbf{0.2} (0.4) & 2.0 (0.63) & 1.6 (0.49) \\
\hline

$T_e$ (s)
& \textbf{96.6} (0.07) & 114.4 (0.49) & 104.0 (4.23) 
& \textbf{121.7} (0.32) & 137.4 (1.05) & 133.7 (0.37) 
& \textbf{148.1} (0.66) & 154.0 (1.46) & 235.8 (1.71) 
& 256.2 (3.08) & 198.9 (2.74) & \textbf{182.97} (1.20) 
& \textbf{274.4} (0.33) & 314.9 (2.28) & 285.9 (1.40)\\
\bottomrule

\end{tabular}
}
\end{table*}
\section{Evaluation}\label{sec:evaluation}


We investigate the efficacy, scalability, and robustness of ReloPush in simulations and hardware experiments (\figref{fig:exp_setup}).

\subsection{Implementation}


\textbf{Experimental Setup}. We implement our framework on MuSHR~\citep{srinivasa2019mushr}, an open-source 1/10th-scale mobile robot racecar, augmented with a 3D-printed flat bumper for pushing (see~\figref{fig:real}), similar to the one used by~\citet{talia2023pushr}, deployed in a workspace of area ${4 \times 5.2} m^2$. Simulations do not involve physics computations -- insights about the pushing performance can be extracted from our real-world experiments. Across simulations and hardware experiments, we assume access to accurate robot localization (in real experiments, we make use of an overhead motion-capture system). We use objects of cubic shape with a side of $0.15m$ and a mass of 0.44 kg. The friction coefficient on the bumper-object surface was measured to be ${\sim 0.73}$. 

\begin{figure}
\begin{subfigure}{0.24\textwidth}
\includegraphics[width=\linewidth]{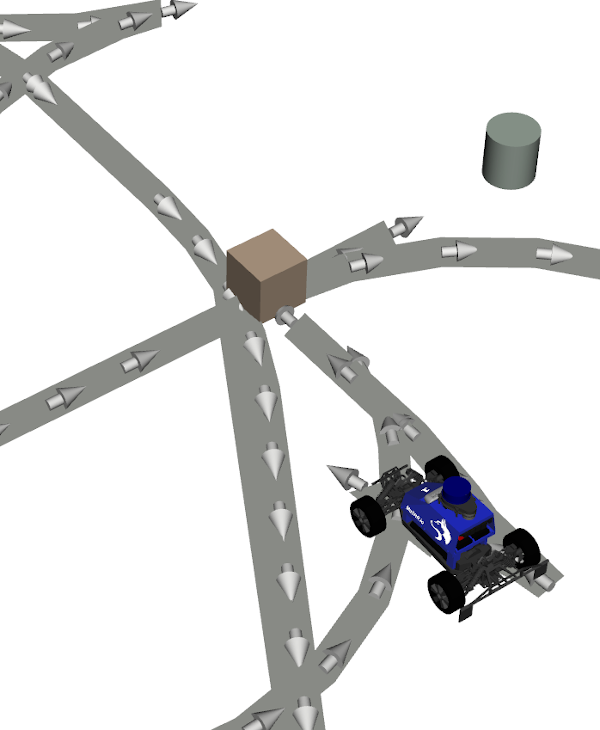}
\caption{Simulated trial.\label{fig:sim}}
\end{subfigure}
\begin{subfigure}{0.24\textwidth}
\includegraphics[width=\linewidth]{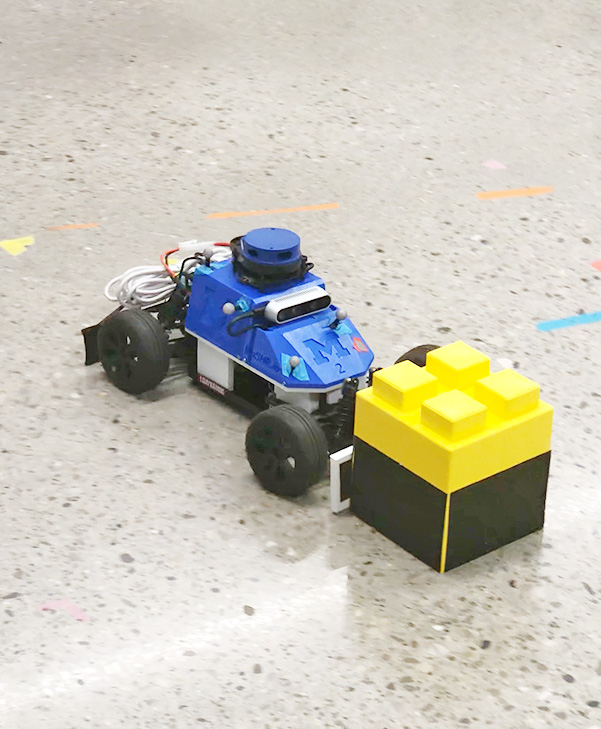}
\caption{Real-robot trial.\label{fig:real}}
\end{subfigure}
  \caption{Stills from simulated (\subref{fig:sim}) and real robot (\subref{fig:real}) experiments.\label{fig:exp_setup}}
\end{figure}




\textbf{Software}. We implement our framework using the Open Motion Planning Library for Dubins path planning~\citep{sucan2012open}, and the code of~\citet{wen2022cl} for Hybrid A$^*$ planning. Across simulated and real-world experiments, we use a Receding Horizon Controller (RHC) based on the implementation of the MuSHR~\citep{srinivasa2019mushr} ecosystem. We run graph construction and search using Boost Graph Library~\citep{boost_graph}. For searching the PT-Graph, we use Dijkstra's algorithm but any graph search algorithm can be used. All planning experiments are conducted on a desktop equipped with an Intel Core i7-13700 CPU and 32G RAM. We share our software implementation online at \href{https://github.com/fluentrobotics/ReloPush}
{\textbf{https://github.com/fluentrobotics/ReloPush}}.


\textbf{Metrics}. We evaluate our system with respect to the following metrics:
\begin{itemize}
    \item S: Success rate -- a trial is successful if a planner successfully finds a feasible rearrangement sequence.
    \item $T_p$: The time it takes to compute a complete rearrangement plan.
    \item $L_t$: The total length of the path that the robot travelled, including the reaching and pushing segments.


    \item $N_{loss}$: The total number of objects the robot lost contact with.

    \item $T_e$: The total time takes to execute a complete rearrangement plan.
\end{itemize}

We also extract insights on the planning behavior of all algorithms using the following indices:
\begin{itemize}
    \item $N_{pre}$: The total number of objects prerelocated to a feasible starting pushing pose (see~\figref{fig:example}).
    \item $N_{obs}$: The total number of removed blocking objects (see~\figref{fig:tgraph}).
    \item $L_p$: The total length of path segments involving pushing.
\end{itemize}

\textbf{Baselines}. We compare the performance of ReloPush against two baselines:
\begin{itemize}
\item \textsc{No-PreRelo (NPR)}: a variant of \textsc{ReloPush} that also uses the PT-graph to handle non-monotone cases but does not plan \emph{prerelocations}. Instead, it invokes Hybrid A$^*$ if a collision-free Dubins path is found, to add edges.



\item \textsc{MP}: a variant of our system that does not make use of the PT-graph at all but rather invokes Hybrid A$^*$ to plan a sequence of rearrangement tasks in a greedy fashion, and thus can only handle monotone cases.
\end{itemize}

\textbf{Experimental Procedure.} We consider a series of rearrangement scenarios of varying complexity (see~\figref{fig:scenarios}) instantiated in simulation and the real world. To extract statistics on planning performance, we instantiated 100 trials of each scenario by locally randomizing the start and goal positions of objects within a range of $\pm 0.05m$ around the nominal instances of~\figref{fig:scenarios}. To evaluate the robustness of our complete architecture, we executed the same scenarios in a physical workspace on a real MuSHR~\citep{srinivasa2019mushr} robot. To ensure fairness in real-robot experiments, we chose instances where all algorithms were successful in planning. We ran each scenario 5 times per algorithm.








\begin{figure}
    \captionsetup{skip=0pt}
    \centering
    \includegraphics[width=\linewidth]{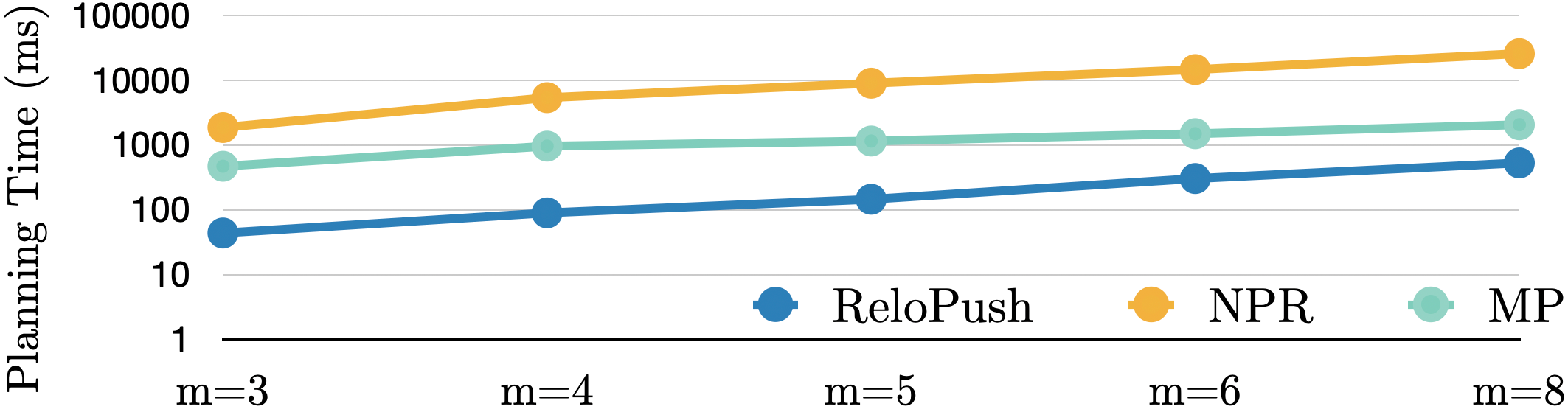}
    \caption{Average planning time across scenarios, shown in logarithmic scale. ReloPush scales well with the number of objects compared to baselines.}
    \label{fig:sim_plan_time_res} 
\end{figure}

\subsection{Results}


\textbf{Planning Performance.} ReloPush dominates baselines in success rate and planning time (see~\tabref{tab:sim-results}, \figref{fig:sim_plan_time_res}), with the gap becoming more pronounced as the clutter (number of objects) increases. This happens because increased clutter is more likely to lead to non-monotone instances. For example, most $m=6$ instances are non-monotone because two objects overlap with goals of other objects. Since \textsc{MP} can only handle monotone rearrangements, it fails more frequently in these harder instances. It is also worth observing that kinematic constraints make some instances harder to solve \emph{regardless} of the number of objects. For example, some of the $m=3$ instances are challenging because $o^s_2$ is situated so close to its goal $o^g_2$ that the optimal path connecting them goes out of the boundary. In contrast, \textit{ReloPush} handled this scenario effectively via \textit{prerelocation}. \tabref{tab:planning-behavior} provides intuition on the planning decisions that ReloPush made enabled increased performance. We see that as clutter increases, ReloPush makes increasingly more workspace modifications (prerelocations and blocking-object removals).

\begin{figure}
\begin{subfigure}{.485\linewidth}
\includegraphics[width=\linewidth]{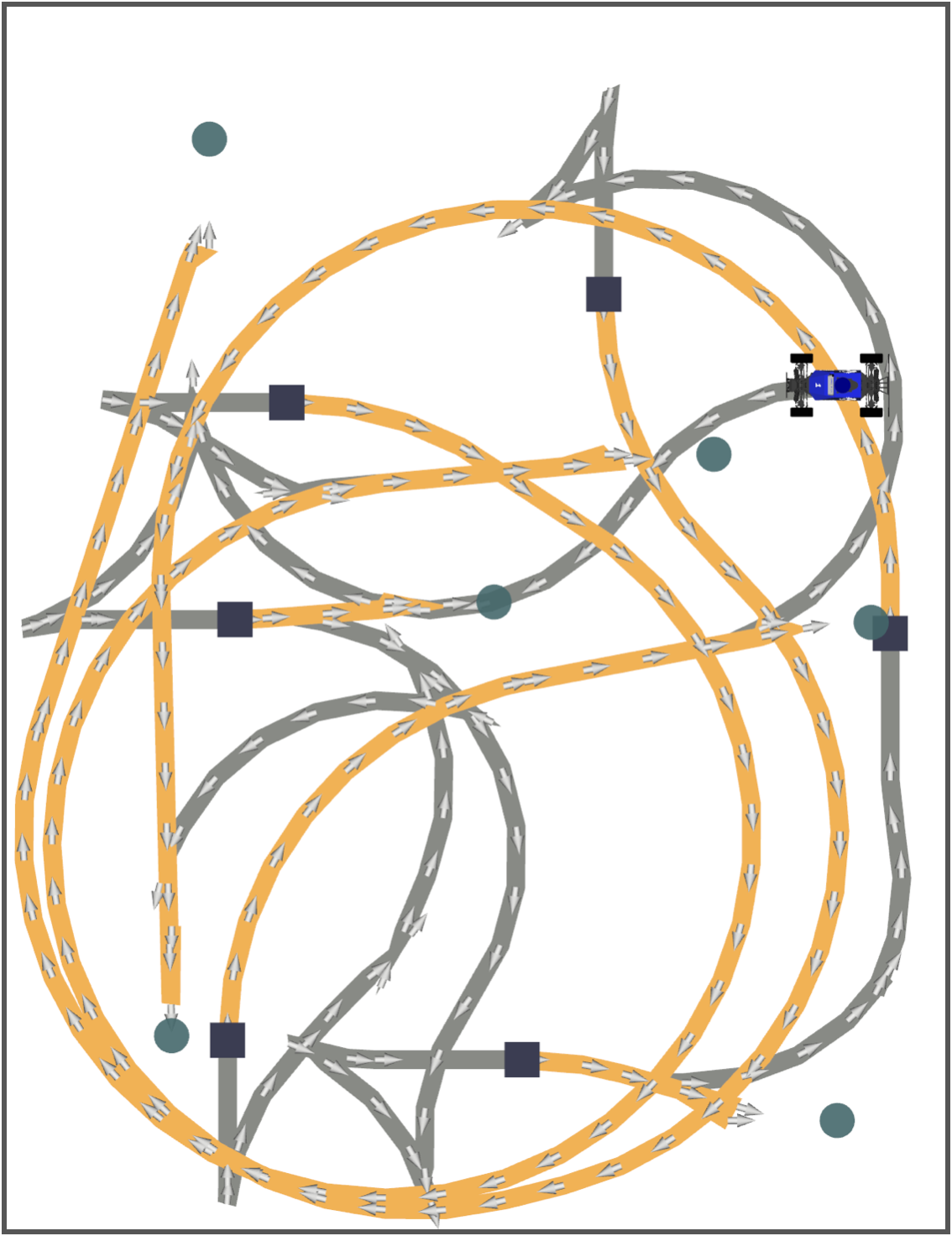}
\caption{\textsc{MP}\label{fig:path_mp}}
\end{subfigure}
~
\begin{subfigure}{0.485\linewidth}
\includegraphics[width=\linewidth]{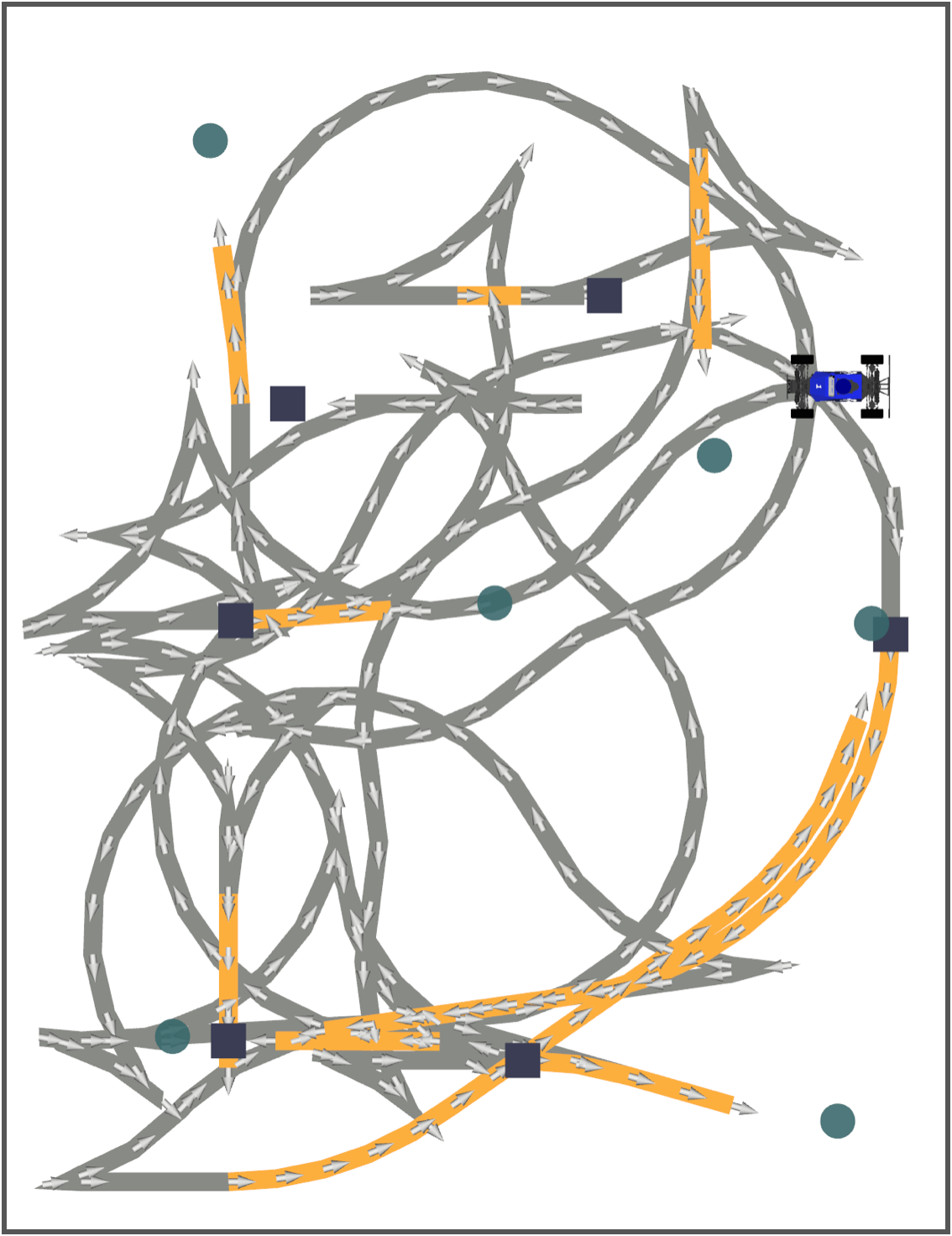}
\caption{\textsc{ReloPush}\label{fig:path_proposed}}
\end{subfigure}
\caption{Paths generated by \textsc{MP} (\subref{fig:path_mp}) and \textsc{ReloPush} (\subref{fig:path_proposed}) for the $m=6$ scenario. Squares and circles represent respectively start and goal object poses. Continuous lines represent planned paths with pushing segments shown in yellow. \textsc{ReloPush} plans substantially shorter pushing segments to minimize the risk of losing contact with an object during execution.\label{fig:path_compare}}
\end{figure}

\textbf{Real Robot Experiments.} ReloPush successfully completed all trials, dominating baselines in terms of execution time (see~\tabref{tab:real_result}). ReloPush never lost contact with any objects, in contrast to baselines. One reason for that is that ReloPush maintains low pushing path length (for better or similar total path length) as shown in~\tabref{tab:sim_result_relo}. The shorter the pushing distance, the lower the risk of losing the object due to model errors and uncertainties. By accounting for this during planning through prerelocations (see~\figref{fig:path_compare}), ReloPush reduces the burden on the path tracking controller which will inevitably accrue errors during execution (the same path tracking controller was used for all algorithms). A video with footage from our experiments can be found at~\href{https://youtu.be/_EwHuF8XAjk}{https://youtu.be/\_EwHuF8XAjk}.








\section{Limitations}\label{sec:discussion}

While ReloPush is capable of handling densely cluttered problem instances, it assumes that at least one object is initially accessible by the pusher. Future work involves enabling the pusher to ``break'' cluttered configurations through impact to make space for planning. To achieve object stability during pushing, we used sandpaper to increase friction at the pusher-object contact, and locked the pusher's steering below the quasistatic limit, which further complicated planning. Ongoing work involves learning a model of push dynamics to relax planning and close the loop for push stability during execution. ReloPush could be further improved by optimizing the use of space when planning \emph{prerelocations} and when removing blocking objects. Extensions to this work will study scenarios involving the rearrangement of unstructured objects like debris, and construction materials.










\balance
\bibliographystyle{abbrvnat}
\bibliography{references}



\end{document}